\documentclass[conference]{IEEEtran}
\IEEEoverridecommandlockouts
\usepackage{cite}
\usepackage{amsmath,amssymb,amsfonts,amsthm}
\usepackage{algorithm}
\usepackage{algpseudocode}
\usepackage{graphicx}
\usepackage{textcomp}
\usepackage{xcolor}
\usepackage{derivative}
\newtheorem{theorem}{Theorem}
\newtheorem{lemma}{Lemma}

\usepackage[square,numbers]{natbib}
\def\BibTeX{{\rm B\kern-.05em{\sc i\kern-.025em b}\kern-.08em
    T\kern-.1667em\lower.7ex\hbox{E}\kern-.125emX}}
\begin{document}

\title{From Relative Entropy to Minimax: A Unified Framework for Coverage in MDPs \vspace*{-0.2in}
\thanks{ONR, NSF TILOS AI Institute, the UCSD Centers for Machine intelligence, computing, and security (MICS).}
}

\author{\IEEEauthorblockN{Xihe Gu}
\IEEEauthorblockA{\textit{University of California, San Diego} \\
San Diego, USA \\
x9gu@ucsd.edu}
\and
\IEEEauthorblockN{ Urbashi Mitra}
\IEEEauthorblockA{\textit{University of Southern California} \\
Los Angeles, USA \\
ubli@usc.edu}
\and
\IEEEauthorblockN{ Tara Javidi}
\IEEEauthorblockA{\textit{University of California, San Diego} \\
San Diego, USA \\
tjavidi@ucsd.edu}

}

\maketitle

\begin{abstract}
Targeted and deliberate exploration of state--action pairs is essential in reward-free Markov Decision Problems (MDPs). More precisely, different state-action pairs exhibit different degree of importance or difficulty which must be actively and explicitly built into a controlled exploration strategy. 
% To this end, a weighted and parameterized family of concave coverage functions $U_\rho$ is proposed and optimized directly over state--action occupancy measures. The proposed objective function unifies several widely studied objectives within a single framework: divergence-based marginal matching, weighted average coverage, and worst-case (minimax) coverage. 
To this end, we propose a weighted and parameterized family of concave coverage objectives, denoted by $U_\rho$, defined directly over state--action occupancy measures. This family unifies several widely studied objectives within a single framework, including divergence-based marginal matching, weighted average coverage, and worst-case (minimax) coverage.
While the concavity of $U_\rho$ captures the diminishing return associated with over-exploration, the simple closed form of the gradient of $U_\rho$ enables an explicit control to prioritize under-explored state--action pairs. 
% The corresponding gradient-based algorithm is then proposed exploiting this structure of $U_\rho$. More specifically, the proposed strategy steer the induced occupancy toward a desired coverage pattern. Furthermore, it is shown that, as $\rho$ grows, the proposed strategy increasingly prioritizes the least-explored state-action pairs. 
Leveraging this structure, we develop a gradient-based algorithm that actively steers the induced occupancy toward a desired coverage pattern. Moreover, we show that as $\rho$ increases, the resulting exploration strategy increasingly emphasizes the least-explored state--action pairs, recovering worst-case coverage behavior in the limit.

%The proposed functional family $U_\rho$ and the associated algorithm provide a principled and flexible foundation for designing exploration strategies in reward-free and general reinforcement learning settings.

\end{abstract}


\section{Introduction}
Exploration is a fundamental problem in reinforcement learning (RL) and sequential decision making. 
A natural perspective is to view exploration as the problem of shaping the distribution of sampled state--action pairs (e.g.,~\cite{coverage_1,coverage_3,coverage_4,BJM25}).  However, existing work along these lines is mostly reward-based.
We design and analyze methods for \emph{reward-free exploration}~\cite{jin2020reward, kaufmann2021adaptive, menard2021fast},
wherein the environment is explored independently of any downstream task.
Such settings arise naturally in model learning, system identification, and
pretraining, where the primary goal is not immediate performance but effective
\emph{coverage} of state--action space.  In reward-free Markov Decision Processes (MDPs), exploration is often constrained by limited interaction budgets and different state--action pairs may exhibit different levels of statistical/learning complexity.

Uncertainty can be decreased by balancing the  number of visits to each state-action pair leading to distribution-based exploration methods. In contrast to reward optimization, the coverage objective is achieved by balancing the \emph{occupancy measure}, i.e. the long-run visitation frequencies induced by a policy, among the state-action pairs. In \citet{hazan2019provably},  the entropy of the induced visitation distribution is maximized leading to uniform long-run coverage under suitable conditions. This admits a convex objective function over occupancy measures, and comes with strong theoretical guarantees. Related ideas appear in entropy-regularized control~\cite{tiapkin2023fast, tarbouriech2020active, zhang2021exploration}, state-marginal matching~\cite{lee1906efficient, al2023active}, and ergodic exploration~\cite{kaelbling1993learning,singh2009rewards,singh2010intrinsically,zheng2018learning, tarbouriech2019active}.

Existing works typically hard-code a single exploration objective, most often entropy or a divergence, making it difficult to reason systematically about alternative coverage criteria. In contrast, our work unifies these objectives within a parameterized family of concave utility functions defined over state–action occupancies. The concavity of this family captures the diminishing return in exploration: additional visits to already well-explored state–action pairs yield progressively smaller marginal benefits. Furthermore, the weights allow the exploration to be prioritized heterogeneously. % Separately, 
The use of weights allows the framework to encode heterogeneous importance or uncertainty across state–action pairs.  Together, these ingredients yield a flexible formulation that subsumes existing coverage objectives, including entropy-based, weighted-average, and worst-case coverage objectives. Furthermore, it is consistent with an \emph{active learning} perspective (e.g.,~\cite{shekhar2020adaptive,Kartik_NM22}), in which data collection decisions are chosen adaptively based on past observations to explicitly shape the future visitation distribution according to a prescribed coverage objective.

The contributions of this work are as follows:
\begin{itemize}
    \item We propose a general and unified parameterized family of concave coverage objective function of occupancy measure, which admits tractable active optimization, exploration, and learning.
    % \item We show that several widely studied objectives arise as special cases, including KL-divergence, weighted average coverage, and max--min (worst-case) coverage objectives relevant for estimation and robustness yielding a unified framework for reward-free exploration.
    \item We show that several widely studied objectives arise as special cases of our framework, including KL-divergence, weighted average coverage, and minimax (worst-case) coverage objectives relevant to estimation and robustness.
    \item We analyze the geometric and analytical properties of these objectives as we vary their unifying parameter. As a special case, we demonstrating how a single parameter interpolates between average and worst-case coverage.
    \item We propose a family of active, gradient-based algorithms to steer the induced occupancy toward a desired coverage pattern, explicitly accounting for different coverage criteria. 
    % \item We empirically evaluate our framework on representative environments and demonstrate that it consistently outperforms existing coverage-based and reward-free exploration baselines.
\end{itemize}

\section{Preliminaries}

\subsection{Controlled Markov Chain Model}

We consider a controlled Markov chain with a finite state space
$\mathcal{S}$ of size $S$ and a finite action space $\mathcal{A}$ of size $A$.
At each time step $t$, the agent observes a state $s_t \in \mathcal{S}$,
selects an action $a_t \in \mathcal{A}$, and transitions to a next state
$s_{t+1}$ according to the transition kernel
$P(s' \mid s,a)$.

\subsection{Policies}

A policy $\pi$ specifies how actions are selected by the agent.
In general, a policy may depend on the entire interaction history.
In this work, we restrict attention to \emph{stationary randomized policies},
which map states to distributions over actions,
\[
\pi : \mathcal{S} \to \Delta(\mathcal{A}),
\]
where $\Delta(\mathcal{A})$ denotes the probability simplex over action space $\mathcal{A}$.
We denote the set of all stationary policies by $\Pi$.  

\subsection{Induced Occupancy Measures}

Executing a stationary policy $\pi$ induces an invariant joint distribution over state--action pairs. Furthermore, we assume each stationary policy induce an ergodic Markov chain with a unique stationary distribution.
We denote by $d^\pi_{s,a}$ the (stationary) state--action occupancy measure
associated with policy $\pi$, defined as
\[
d^\pi_{s,a} := \Pr(s_t = s, a_t = a \mid \pi),
\]
where the probability is taken under the stationary distribution of the resulting Markov chain.
The occupancy measure $d^\pi$ captures the long-run visitation frequency of each state--action pair and is the primary object of interest in our coverage formulation.
We denote by $\mathcal{D} := \{ d^\pi : \pi \in \Pi \}$ the set of all
achievable (invariant/stationary) occupancy measures.

% \subsection{Mixtures of Stationary Policies}

% Given a finite collection of stationary policies
% $C = (\pi_0, \ldots, \pi_{K-1})$ and a mixing distribution
% $\beta \in \Delta_K$, we define a \emph{mixture policy}
% $\pi_{\mathrm{mix}} = (\beta, C)$ as follows:
% at each time step, a policy $\pi_k$ is selected with probability $\beta_k$,
% and the action is drawn according to $\pi_k$.

% The occupancy measure induced by the mixture policy is given by the convex
% combination
% \begin{equation}
% d^{\pi_{\mathrm{mix}}}_{s,a}
% =
% \sum_{k=0}^{K-1} \beta_k \, d^{\pi_k}_{s,a}.
% \end{equation}
% Thus, mixtures of stationary policies induce occupancy measures lying in
% the convex hull of the occupancy measures of the constituent policies.

\subsection{Exploration Objective}

 We consider the general problem of optimizing coverage where the objective is designed as a decomposable concave functional over the state-action occupancy measure, i.e. the long-run visitation frequencies induced over
state--action pairs $(s,a) \in \mathcal{S}\times\mathcal{A}$:

\[
\pi^* \in \arg\max_{\pi\in \Pi} U_{\rho}(d^\pi) := \arg\max_{\pi\in \Pi}
\sum_{s,a} g_{\rho}(s,a,d^\pi_{s,a}).
\]

In the next section, we introduce a specific parameterized family of objectives within this class and study their analytical and algorithmic properties.

\section{A General Class of Coverage Objectives}
We are given a collection of \emph{coverage weights}
\[
\mu := (\mu_{s,a})_{(s,a)\in\mathcal{S}\times\mathcal{A}} \in \mathbb{R}_{+}^{SA},
\]
where $\mathbb{R}_{+}$ denotes the positive reals.
We assume $\mu$ is finite and bounded, and define
\[
\mu_{\max} := \max_{s,a}\mu_{s,a} < \infty,
\qquad
\mu_{\min} := \min_{s,a}\mu_{s,a} > 0.
\]
The weights $\mu_{s,a}$ encode the relative importance (or difficulty) of
covering each state--action pair, and will remain fixed throughout.

Our goal is to design exploration policies whose induced occupancies achieve
high coverage with respect to the prescribed weighting $\mu$.
To this end, we introduce a general exploration utility 
functional (\emph{coverage objective}) on state--action 
pairs  as a family of concave functions parameterized by a common scalar $\rho \ge 1$, %$g_{\rho} : \mathcal{D} \to \mathbb{R}$
\begin{equation} \label{eq:U_rho_def}
\begin{split}
    g_{\rho}(s, a, d_{s, a})&:= 
    \begin{cases}
    \displaystyle \mu_{s, a}\, \log d_{s, a}, & \rho = 1, \\[1.2em]
    \displaystyle \frac{\mu_{s, a}^{\rho}}{1-\rho }\, d_{s, a}^{\,1-\rho}, & \rho > 1,~
    \end{cases} \\
    \text{and} \qquad U_{\rho}(d) &:= \sum_{s,a} g_{\rho}(s,a,d_{s,a}).
\end{split}
\end{equation}

The key property of $U_\rho$ is that its gradient admits a simple closed form:
\begin{equation}
    \label{eq:U_gradient}
    \bigl(\nabla U_\rho(d)\bigr)_{s,a}
    =
    \left(\frac{\mu_{s,a}}{d_{s,a}}\right)^{\rho}.
\end{equation}

Several prior works~\cite{shekhar2020adaptive,tarbouriech2019active,tarbouriech2020active} emphasize the importance of the ratio $\mu_{s,a}/d_{s,a}$, which measures how well a state--action pair $(s,a)$ is explored relative to its prescribed importance $\mu_{s,a}$. 
The gradient of $U_\rho$ reflects a diminishing-returns effect: state–action pairs that are already well covered (i.e., with large $d_{s,a}$ relative to $\mu_{s,a}$) induce a gradient close to zero, whereas poorly covered pairs with small $d_{s,a}$ produce large gradients.
Consequently, the gradient of $U_\rho$ places increasing pressure on under-explored state--action pairs. 
The parameter $\rho$ controls the degree of this effect: larger values of $\rho$ amplify the disparity between gradient magnitudes, thereby focusing the optimization more strongly on the most under-covered regions.

\section{Analysis: From Relative Entropy to Minimax Exploration}

The family of coverage objectives $\{U_\rho\}_{\rho \ge 1}$ provides a continuous spectrum of exploration criteria.
In this section, we analyze several representative instances of $U_\rho$ and show how they recover widely used coverage objectives, ranging from  Kullback-Leiber (KL) -based matching to worst-case (minimax) exploration.

\subsection{Special case: $\rho = 1$ (KL divergence)}
Cross-entropy and KL-divergence objectives have been widely used in MDPs and RL~\cite{lee1906efficient,ho2016generative, schulman2015trust, mutti2023convex}, particularly for occupancy matching, imitation learning, and entropy-regularized control. We show that this classical objective arises naturally as a special case of our coverage family when $\rho = 1$.
\begin{lemma}
\label{lem:rho1_KL}
Define the normalized weights
\[
\bar\mu_{s,a}
:=
\frac{\mu_{s,a}}{\sum_{s',a'} \mu_{s',a'}}.
\]
When $\rho = 1$, maximizing the coverage objective
\[
U_1(d) = \sum_{s,a} \mu_{s,a}\,\log d_{s,a}
\]
over $d$ is equivalent to minimizing the KL divergence
$\mathrm{KL}(\bar\mu \,\|\, d)$.
\end{lemma}

\begin{proof}
By definition,
\[
U_1(d)
=
\sum_{s,a} \mu_{s,a}\,\log d_{s,a}
=
\Bigl(\sum_{s,a} \mu_{s,a}\Bigr)
\sum_{s,a} \bar\mu_{s,a}\,\log d_{s,a}.
\]
The second term is the negative cross entropy between $\bar\mu$ and $d$, i.e.,
\[
\sum_{s,a} \bar\mu_{s,a}\,\log d_{s,a}
=
-\mathrm{CE}(\bar\mu \,\|\, d).
\]
Using the identity
\(
\mathrm{CE}(\bar\mu \,\|\, d)
=
\mathrm{KL}(\bar\mu \,\|\, d) + H(\bar\mu),
\)
we obtain
\[
U_1(d)
=
-\Bigl(\sum_{s,a} \mu_{s,a}\Bigr)
\bigl(\mathrm{KL}(\bar\mu \,\|\, d) + H(\bar\mu)\bigr).
\]
Since $\mu$ (and hence $\bar\mu$) is fixed, $H(\bar\mu)$ is constant with respect
to $d$.
Therefore, maximizing $U_1(d)$ over $d$ is equivalent to minimizing
$\mathrm{KL}(\bar\mu \,\|\, d)$.
\end{proof}

\subsection{Special case: $\rho = 2$ (Average relative coverage)}
Several prior works~\cite{tarbouriech2019active, tarbouriech2020active, shekhar2020adaptive,de2024global} design exploration objectives that optimize average-case criteria in MDPs, including average visitation, average uncertainty, and average model estimation error across state–action pairs. We show that such objectives arise naturally from our framework when $\rho = 2$.

\begin{lemma}[Average relative coverage]
\label{lem:rho2_avg}
When $\rho = 2$, the coverage objective reduces to
\[
U_2(d)
=
-\sum_{s,a} \mu_{s,a}\,\frac{\mu_{s,a}}{d_{s,a}}
=
-\mathbb{E}_{(s,a)\sim\bar\mu}
\!\left[\frac{\mu_{s,a}}{d_{s,a}}\right],
\]
where $\bar\mu$ denotes the normalized version of $\mu$.
Maximizing $U_2(d)$ penalizes state--action pairs in proportion to their average
relative under-coverage with respect to $\mu$.
\end{lemma}

In particular, if we have $\mu_{s, a} = \sqrt{c_{s, a}}$, where $c_{s, a} = \sum_{s'} P(s' \mid s,a)(1 - P(s' \mid s,a))$ measures the variance of the transition distribution.
In this case, maximizing $U_2 = -\sum_{s, a} \frac{c_{s, a}}{d_{s, a}} = -n\sum_{s, a}\mathbb{E}[D_{l_2}(P(\cdot|s, a), \widehat{P}(\cdot|s, a)] $ , where $n$ is the total number of samples, is the equivalent to minimizing the average expected estimation error of the transition distributions, as studied in \cite{shekhar2020adaptive}.

\subsection{Special case: $\rho \to \infty$ (Worst-case coverage)}
Worst-case (minimax) objectives are widely used in RL, control, and learning theory, as they provide robust guarantees by explicitly controlling the least-covered or most uncertain components of the system~\cite{shekhar2020adaptive, shekhar2021adaptive, tarbouriech2020active, al2023active, halim2025fairness}.

For $\rho>1$,
$(1-\rho)U_\rho(\mu,d) = \sum_{s,a} d_{s,a}
\left(\frac{\mu_{s,a}}{d_{s,a}}\right)^\rho$.

Then,
\begin{equation}
\label{eq:Urho_limit}
\lim_{\rho \to \infty}
\Bigl((1-\rho)U_\rho(\mu,d)\Bigr)^{1/\rho}
=
\max_{s,a} \frac{\mu_{s,a}}{d_{s,a}}.
\end{equation}
As a consequence, maximizing $U_\rho(\mu,d)$ for large $\rho$,
approximates the minimax coverage objective
\begin{equation}
\label{eq:minimax}
\min_{d \in \mathcal{D}}
\max_{s,a}
\frac{\mu_{s,a}}{d_{s,a}}.
\end{equation}
We can formalize this approximation as follows.
\begin{lemma}\label{lem:limit_Urho}
%The following is the proof of \eqref{eq:Urho_limit}
Let $i$ index state--action pairs $(s,a)$ and define
$r_i := \mu_i / d_i$ and $r_{\max} := \max_i r_i$. If we define $V_\rho := \Bigl( (1-\rho)U_\rho(\mu,d) \Bigr)^{1/\rho}$, then
\begin{eqnarray*}
\liminf_{\rho\to\infty} V_\rho  & = &  r_{\max}.
\end{eqnarray*}
\end{lemma}
\begin{proof}
For $\rho>1$, we can rewrite
\[
(1-\rho)U_\rho(\mu,d)
=
\sum_i \mu_i^\rho d_i^{1-\rho}
=
\sum_i d_i \, r_i^\rho.
\]
%Define
%$V_\rho := \Bigl( (1-\rho)U_\rho(\mu,d) \Bigr)^{1/\rho}
%= \Bigl( \sum_i d_i \, r_i^\rho \Bigr)^{1/\rho}$.

Let $r_{\max} := \max_i r_i$. Since $r_i^\rho \le r_{\max}^\rho$ for all $i$, we obtain the upper bound
\[
\sum_i d_i r_i^\rho \le r_{\max}^\rho \sum_i d_i = r_{\max}^\rho,
\]
which implies $V_\rho \le r_{\max}$.

For the lower bound, let $i^\star$ be any index attaining the maximum, i.e.,
$r_{i^\star} = r_{\max}$. Since $d_{i^\star} > 0$, we have
\[
\sum_i d_i r_i^\rho \ge d_{i^\star} r_{\max}^\rho,
\]
and therefore
\[
V_\rho \ge d_{i^\star}^{1/\rho} r_{\max}.
\]
Taking the limit $\rho \to \infty$ yields $d_{i^\star}^{1/\rho} \to 1$, and hence
\[
\liminf_{\rho\to\infty} V_\rho \ge r_{\max}.
\]

\noindent Combining the upper and lower bounds concludes the proof.  
\end{proof}

Another view point is that $U_\rho$ has the property that for every $d = (d_{1,1},..., d_{S,A}) \in D$, if $\frac{\mu_{s_i, a_i}}{d_{s_i, a_i}} > \frac{\mu_{s_j, a_j}}{d_{s_j, a_j}}$,
$$\lim_{\rho \to \infty} \frac{\dot{U}_\rho(d_{s_i, a_i})}{\dot{U}_\rho(d_{s_j, a_j})}
=
\lim_{\rho \to \infty}
\left(
\frac{\mu_{s_i,a_i} / d_{s_i,a_i}}{\mu_{s_j,a_j} / d_{s_j,a_j}}
\right)^{\rho}
= \infty.
$$

This shows that, as $\rho$ increases, $U_\rho$ places increasingly dominant weight on the entry that has largest value of $\mu_{s,a}/d_{s,a}$.
Consequently, minimizing $U_\rho$ asymptotically prioritizes minimizing the maximum ratio $\mu_{s,a}/d_{s,a}$, which motivates the following result.
\begin{lemma}
\label{lem:rho_to_infty_minimax}
Fix $\mu$ and let $\mathcal{D}$ denote the feasible set of occupancy measures.
For each $\rho\ge 1$, let
\[
d_\rho^\star \in \arg\min_{d\in \mathcal{D}} U_\rho(d),
\]
and define the minimax optimizer
\begin{equation}
d_{\min\max}^\star \in
\arg\min_{d\in \mathcal{D}}
\max_{(s,a)\in [S]\times[A]}
\frac{\mu_{s,a}}{d_{s,a}}.
\label{eq:minimax_problem}
\end{equation}
Then every limit point of the sequence $\{d_\rho^\star\}_{\rho\ge 1}$ as
$\rho\to\infty$ is a solution of~\eqref{eq:minimax_problem}.
In particular, if the minimizer in~\eqref{eq:minimax_problem} is unique, then
\[
d_\rho^\star \to d_{\min\max}^\star
\qquad \text{as } \rho\to\infty.
\]
\end{lemma}

The proof follows the same argument as Theorem~8 in~\cite{xiasome} and is therefore omitted.
 % Therefore, $\rho$ is a hyperparameter that we can tune to make the solution of $U_\rho$ arbitrarily close to $d_{\min\max}^\star$.
Therefore, $\rho$ acts as a tunable hyperparameter that smoothly interpolates the solution of $U_\rho$ between the average-case and worst-case solutions.

\section{Optimization Algorithm}
We search over the class of stationary policies to solve the following problem
\begin{equation}
    \pi^* \in \arg\max_{\pi \in \Pi}  U_\rho(d^\pi)
\end{equation}

\subsection{State Distribution Estimate}
For each time $t = 1, ..., n$, we collect the sampling counts into a vector
\[
T(t) = \bigl( T_{s,a}(t) \bigr)_{s \in \mathcal{S},\, a \in \mathcal{A}}
\in \mathbb{N}_0^{S \times A},
\]
where $\mathbb{N}_0 = \{0,1,2,\ldots\}$.
Let $T_{s,a}(t)$ be the number of times the learner has executed action $a$ in state $s$ before time $t$. Thus, $\sum_{s, a}{T_{s,a}(t)} = t-1$.

The empirical sampling distribution is
\[
\widehat{d}_{s,a}(t) := \frac{T_{s,a}(t)}{t},
\qquad s \in \mathcal{S},\ a \in \mathcal{A}.
\]

\subsection{Smoothed Optimization}
The gradient of $U_\rho$ becomes unbounded when some occupancy $d_{s,a}$ approaches zero. 
To ensure well-behaved gradients and enable smooth optimization, we restrict the feasible domain away from the boundary of the simplex.

For any $\underline{\eta} \in (0, 1/2)$, define the \emph{restricted occupancy set}
\[
\mathcal{D}_{\underline{\eta}}
:=
\Big\{d\in\mathcal{D}:\ d(s,a)\ge 2\underline{\eta},\ \ \forall(s,a)\in\mathcal{S}\times\mathcal{A}\Big\}.
\]

\begin{lemma}\label{lem:smooth_Urho}
$U_\rho(\mu,\cdot)$ is $C_{\underline{\eta}}$-smooth on $\mathcal{D}_{\underline{\eta}}$, i.e.,
\[
\|\nabla_d U_\rho(\mu,d) - \nabla_d U_\rho(\mu,d')\|_2
\le
C_{\underline{\eta}} \, \|d - d'\|_2,
 \forall d,d' \in \mathcal{D}_{\underline{\eta}}.
\]
The smoothness constant is given by
\[
C_{\underline{\eta}}
\;=\;
\frac{\rho\, \mu_{\max}^{\rho}}{2^{\rho+1}\,\underline{\eta}^{\rho+1}},
\qquad
\mu_{\max} := \max_{s,a} \mu_{s,a}.
\]
\end{lemma}
The proof is provided in 
Appendix \ref{appendix: lem_smoothUrho_proof}(see Supplementary Material). This smoothness property enables standard optimization over $\mathcal{D}_{\underline{\eta}}$.

\subsection{Algorithm Design}
We now describe an algorithm for maximizing the coverage objective $U_\rho$. Here, we consider the oracle setting in which the
transition probabilities $P(\cdot \mid s,a)$ are known, and the algorithm can be extended to the unknown transition cases.
The central question is how to allocate sampling effort so that the empirical state--action occupancy $\widehat d$ converges to a maximizer of $U_\rho$.

The algorithm proceeds in an episodic manner. Let $t_k$ denote the time index at which episode $k$ begins, and let $\tau_k$
denote the length of episode $k$.
At the beginning of each episode, it computes a policy that maximizes the current gradient of $U_\rho$.
This policy is then executed for $\tau_k$ steps, producing an updated
occupation measure.
We show that, under this oracle setting, the resulting sequence of occupation measures converges to the maximizer of $U_\rho$.

\begin{algorithm}[h]
\caption{$\rho$-weighted-coverage}
\label{alg:FW-oracle}
\begin{algorithmic}[1]
\State \textbf{Input:} Hyperparameter $\rho$, all transition distributions $P(\cdot \mid s,a)$
\State \textbf{Initialization:} Set $T_{s, a}(1) = 0$
\For{$k = 1,2,\ldots,K-1$} \Comment{K episodes}
    \State $\psi_{k+1}
        = \arg\max_{d \in \mathcal{D}_{\underline{\eta}}} \langle (\frac{\mu_{s, a}}{T^+_{s, a}(t_k)/t_k})^\rho, d \rangle$
    \State $\pi_{k+1}(a \mid s)
        = \dfrac{ \psi_{k+1}(s,a) }
                { \sum_{b \in A} \psi_{k+1}(s,b) }$
    \State Execute $\pi_{k+1}$ for $\tau_k$ steps
    \State Update the state-action occupancy $T(t_{k+1})$
\EndFor
\State \textbf{Return: $T(t_K)$} 
\end{algorithmic}
\end{algorithm}

The direction-finding step in line 4 of Algorithm~\ref{alg:FW-oracle} is a Frank--Wolfe (FW) oracle and can be computed using standard techniques such as linear programming or fixed-point methods.
Equivalently, the update can be written as
\begin{equation} \notag
    \begin{split}
        \psi_{k+1}
        &= \arg\max_{d \in \mathcal{D}_{\underline{\eta}}}
        \langle \nabla U_\rho(\widehat{d}(t_k)), d \rangle \\
        &= \arg\max_{d \in \mathcal{D}_{\underline{\eta}}} \langle (\frac{\mu_{s, a}}{T^+_{s, a}(t_k)/t_k})^\rho, d \rangle
    \end{split}
\end{equation}
where $T^+_{s, a}(t) = \max(T_{s, a}(t), \epsilon)$, with some small $\epsilon>0$ to deal with the cold start.
Let $\nu_{k+1}(s,a)$ denote the number of times action $a$ is taken in state $s$
during episode $k$ when executing $\pi_{k+1}$, and define
$
\widehat\psi_{k+1}(s,a) = \nu_{k+1}(s,a)/\tau_k
$
as the empirical visitation frequency within the episode.
The empirical occupancy is then updated as
\begin{equation} \notag
    \begin{split}
        \widehat{d}(t_{k+1})
        &= \frac{\tau_k}{t_{k+1}-1}\,\widehat{\psi}_{k+1}
          + \frac{t_k - 1}{t_{k+1}-1}\,\widehat{d}(t_k) \\
        &= \beta_k \widehat{\psi}_{k+1} + (1-\beta_k)\widehat{d}(t_k),
    \end{split}
\end{equation}
where
$
\beta_k := \tau_k/(t_{k+1}-1)
$
serves as the step size.

\begin{theorem} \label{thm: FW-oracle-regret}
Let episode lengths satisfy 
\[
\tau_k = \tau_1 k^2
\quad\text{and}\quad
t_k = \tau_1 \frac{(k-1)k(2k-1)}{6} + 1,
\]
where $\tau_1$ is the length of the first episode.
% i.e.,
% \[
% t_k = \tau_1 \frac{(k-1)k(2k-1)}{6} + 1
% \quad\text{and}\quad
% \beta_k = \frac{6k}{(k+1)(2k+1)}.
% \]
Algorithm \ref{alg:FW-oracle} satisfies with high probability
\[
U_\rho(\widehat{d}(t_K)) - U(d_\rho^\star)
= \widetilde{\mathcal{O}}(t_K^{-1/3}).
\]
\end{theorem}

\noindent\textit{Proof Outline:}

The proof follows the same outline as the proof of Theorem 1 in~\cite{tarbouriech2019active}. We denote by $\xi_{k+1}$ the approximation error at the end of each episode $k$
(i.e., at time $t_{k+1}-1$). Recalling that 
$\beta_k = \tau_k / (t_{k+1}-1)$, we have
\begin{equation} \notag
    \begin{split}
        \xi_{k+1}
        &= U_\rho(d_\rho^\star) - U_\rho(\widehat{d}_{k+1}) \\
        &= U_\rho(d_\rho^\star) - U_\rho\!\left((1-\beta_k)\widehat{d}_k + \beta_k \widehat{\psi}_{k+1}\right),
    \end{split}
\end{equation}
where $d_\rho^\star \in \arg\min_{d\in \mathcal{D}} U_\rho(d)$, and $\widehat{d}_k$ abbreviates $\widehat{d}(t_k)$.

% Using the $C_{\underline{\eta}}$-smoothness of $U_\rho$ (Lemma~\ref{lem:smooth_Urho}), the optimality of the FW direction $\psi_{k+1}$, and the concavity of $U_\rho$, the
% approximation error satisfies the recursion
% \[
% \xi_{k+1}
% \;\le\;
% (1-\beta_k)\,\xi_k
% \;+\;
% C_{\underline{\eta}}\,\beta_k^2
% \;+\;
% \beta_k\,\Delta_{k+1},
% \]
% where
% \[
% \Delta_{k+1}
% := \bigl\langle \nabla U_\rho(\widehat{d}_k),\,
% \psi_{k+1} - \widehat{\psi}_{k+1} \bigr\rangle
% \]
% captures the discrepancy between the stationary occupancy $\psi_{k+1}$ and its empirical estimate $\hat{\psi}_{k+1}$ for $\tau_k$ steps.

Using the $C_{\underline{\eta}}$-smoothness of $U_\rho$ (Lemma~\ref{lem:smooth_Urho}), we have 
\[
\begin{aligned} \small
\xi_{k+1} 
&\le 
U_\rho(d_\rho^\star) - U_\rho(\widehat{d}_k) 
- \beta_k \langle \nabla U_\rho(\widehat{d}_k),
    \widehat{\psi}_{k+1} - \widehat{d}_k \rangle
  + C_{\underline{\eta}} \beta_k^2
\\[0.4em]
&=
U_\rho(d_\rho^\star) - U_\rho(\widehat{d}_k) 
- \beta_k \langle \nabla U_\rho(\widehat{d}_k),
    \psi_{k+1} - \widehat{d}_k \rangle
  + C_{\underline{\eta}} \beta_k^2 \\
&\qquad \qquad - \beta_k \langle \nabla U_\rho(\widehat{d}_k),
    \widehat{\psi}_{k+1} - \psi_{k+1} \rangle.
% \\[0.4em]
% &\le 
% U_\rho(d_\rho^\star) - U_\rho(\widehat{d}_k) 
% - \beta_k \langle \nabla U_\rho(\widehat{d}_k),
%     d_\rho^\star - \widehat{d}_k \rangle
%   + C_{\underline{\eta}} \beta_k^2 \\
% &\qquad \qquad - \beta_k \langle \nabla U_\rho(\widehat{d}_k),
%     \widehat{\psi}_{k+1} - \psi_{k+1} \rangle
% \\[0.4em]
% &\le 
% (1-\beta_k)\xi_k
% + C_{\underline{\eta}} \beta_k^2
% + \beta_k
% \underbrace{
%    \langle \nabla U_\rho(\widehat{d}_k),
%       \psi_{k+1} - \widehat{\psi}_{k+1}  \rangle
% }_{\Delta_{k+1}},
\end{aligned}
\]
Further with the optimality of the FW direction $\psi_{k+1}$, and the concavity of $U_\rho$, the approximation error satisfies the recursion
\[
\begin{aligned} \small
\xi_{k+1} 
% &\le 
% U_\rho(d_\rho^\star) - U_\rho(\widehat{d}_k) 
% - \beta_k \langle \nabla U_\rho(\widehat{d}_k),
%     \widehat{\psi}_{k+1} - \widehat{d}_k \rangle
%   + C_{\underline{\eta}} \beta_k^2
% \\[0.4em]
% &=
% U_\rho(d_\rho^\star) - U_\rho(\widehat{d}_k) 
% - \beta_k \langle \nabla U_\rho(\widehat{d}_k),
%     \psi_{k+1} - \widehat{d}_k \rangle
%   + C_{\underline{\eta}} \beta_k^2 \\
% &\qquad \qquad - \beta_k \langle \nabla U_\rho(\widehat{d}_k),
%     \widehat{\psi}_{k+1} - \psi_{k+1} \rangle
% \\[0.4em]
&\le 
U_\rho(d_\rho^\star) - U_\rho(\widehat{d}_k) 
- \beta_k \langle \nabla U_\rho(\widehat{d}_k),
    d_\rho^\star - \widehat{d}_k \rangle
  + C_{\underline{\eta}} \beta_k^2 \\
&\qquad \qquad - \beta_k \langle \nabla U_\rho(\widehat{d}_k),
    \widehat{\psi}_{k+1} - \psi_{k+1} \rangle
\\[0.4em]
&\le 
(1-\beta_k)\xi_k
+ C_{\underline{\eta}} \beta_k^2
+ \beta_k
\underbrace{
   \langle \nabla U_\rho(\widehat{d}_k),
      \psi_{k+1} - \widehat{\psi}_{k+1}  \rangle
}_{\Delta_{k+1}}.
\end{aligned}
\]
% where the first step follows from the $C_{\underline{\eta}}$-smoothness of $U_\rho$, proved in Lemma~\ref{lem:smooth_Urho}, 
% the second inequality comes from the FW optimization step and the definition 
% of $\psi_{k+1}$, which gives 
% $\langle \nabla U_\rho(\widehat{d}_k),\, 
% \psi_{k+1} - \widehat{d}_k \rangle 
% \ge 
% \langle \nabla U_\rho(\widehat{d}_k),\, 
% d_\rho^\star - \widehat{d}_k \rangle$,
% the final step follows from the concavity of $U_\rho$. 
The term $\Delta_{k+1}$ captures the discrepancy between the stationary occupancy $\psi_{k+1}$ and its empirical estimate $\widehat{\psi}_{k+1}$ for $\tau_k$ steps. 

Using the concentration bound on the empirical occupancy measure between
$\psi_{k+1}$ and $\widehat{\psi}_{k+1}$, we can upper bound the error term
$\Delta_{k+1}$ as
\[
\Delta_{k+1}
\;\le\;
\frac{c_1}{\sqrt{\tau_k}}
\;+\;
\frac{c_2}{\tau_k},
\]
% where $c_1$ and $c_2$ are constants.
\[
\text{with}\quad
\begin{cases}
c_1 = SA
        \big(\frac{\mu_{\max}}
             {\underline{\eta}} \big)^\rho
       \big(\sqrt{2B/\gamma_{\min}} + \sqrt{\ln(4SA/\delta)/2}\big),\\[1.2em]
c_2 = SA
        \big(\frac{\mu_{\max}}
             {\underline{\eta}} \big)^\rho
       \dfrac{20B}{\gamma_{\min}}.
\end{cases}
\]
Substituting this bound into the recursion yields, $\forall k \ge k_\delta$,
\begin{equation} \label{eq: xi_inter1}
\xi_{k+1}
\;\le\;
(1 - \beta_k)\xi_k
\;+\;
\beta_k \Bigl(
   \frac{c_1}{\sqrt{\tau_k}}
   +
   \frac{c_2}{\tau_k}
\Bigr)
\;+\;
C_{\underline{\eta}} \beta_k^2
\end{equation}
with probability at least $1-\delta$.
Here, the episode index $k_\delta$ is defined such that for all $k \ge k_\delta$,
the empirical occupancy satisfies
$\widehat{d}_{s,a}(t_k) \ge \underline{\eta},
\quad
\forall s \in \mathcal{S},\ a \in \mathcal{A}$.

Choosing the episode length $\tau_k = \tau_1 k^2$ and the bound of the step size
$\beta_k \in [1/k, 3/k]$, the recursion~\eqref{eq: xi_inter1} simplifies to
\begin{equation} \label{eq:xi_inter2}
    \xi_{k+1}
    \;\le\;
    \left(1 - \frac{1}{k}\right)\xi_k
    \;+\;
    \frac{b_\delta}{k^2}.
\end{equation}
where $b_\delta$ absorbs all constant factors.

Solving the recursion in~\eqref{eq:xi_inter2} via a standard telescoping argument yields that, with probability at least $1-\delta$,
\[
\xi_K
\le
\frac{q\xi_q + 2b_\delta \log K}{K}
\le
\frac{\tau_1^{1/3}}{3^{1/3}(t_k-1)^{1/3}}
\left(
q\xi_q + 2b_\delta \log K 
\right).
\]

We therefore obtain the desired high-probability bound.
$$\xi_K = \widetilde{\mathcal{O}}\!\left(\frac{1}{t_K^{1/3}}\right)$$

A detailed proof is provided in Appendix \ref{appendix:proof_oracle_algo}.
This result implies that the error $\xi_K$ converges to zero, and that Algorithm~\ref{alg:FW-oracle} converges to the optimal occupancy measure associated with the coverage objective $U_\rho$ with high probability.

\subsection{Core Design Insight}
This algorithm treats $
\left(\frac{\mu_{s,a}(t_k)}{\widehat{d}_{s,a}(t_k)}\right)^{\rho}
$
as a reward signal at episode $k$, and computes a policy that maximizes the resulting expected cumulative reward, thereby implementing a gradient ascent step on the coverage objective.
This interpretation is natural: Since state--action pairs that are already well explored yield diminishing returns, the algorithm should focus its sampling effort on state--action pairs with
large current ratios $\tfrac{\mu_{s, a}}{d_{s,a}}$.  
The exponent $\rho$ plays a crucial role.  As $\rho$ increases, the disparity between these reward terms amplifies, causing the induced MDP to place progressively more emphasis on those state--action pairs that currently dominate the max expression.  
At the same time, the FW oracle still accounts for the cumulative reward collected along trajectories, ensuring that its intermediate decisions move efficiently toward the worst-case state--action region rather than chasing it myopically.

This combination of global worst-case focus (through the amplified reward structure) and trajectory-level efficiency (through cumulative reward maximization) is precisely what enables our algorithm to achieve strong  performance.

\section{Conclusions}
In this work, we propose a unified framework for coverage in MDPs based on a parameterized family of coverage objectives $U_\rho$ defined over state–action occupancy measures.
Within this framework, KL-based marginal matching, average relative coverage, and worst-case (minimax) coverage all arise as special cases. The parameter $\rho$ provides a continuous mechanism for interpolating between average-case and worst-case notions of coverage, offering a principled way to design exploration objectives tailored to different requirements such as uniformity, weighting, or robustness.

A central theoretical insight of this framework is that the coverage objective $U_\rho$ admits a simple and interpretable gradient structure.
The gradient assigns larger magnitude to state–action pairs that are under-covered relative to their prescribed importance weights, naturally capturing diminishing returns in exploration and directly reflecting the goal of reallocating exploration effort toward poorly covered regions. As $\rho$ increases, the gradient increasingly concentrates on the most under-covered state–action pairs, recovering minimax coverage behavior in the limit.

This gradient structure directly leads to an active, gradient-based algorithm for coverage optimization.
Rather than relying on heuristic intrinsic rewards, the proposed method performs gradient ascent on the coverage objective $U_\rho$, using the gradient as an optimization signal to steer the induced occupancy toward the desired coverage pattern in a principled and analytically tractable manner.

More broadly, our results suggest that coverage in reward-free exploration is best understood through the geometry of the occupancy space and the choice of concave utility functions defined on it.
The framework developed in this paper provides a principled foundation for designing, analyzing, and optimizing coverage objectives beyond entropy, and points toward future extensions to settings with unknown transition dynamics, adaptive importance weights, and integration with general RL frameworks.

\newpage
\bibliography{main}

@String(IJCAI = {IJCAI})

@String(AAAI = {AAAI})

@article{coverage_1,
  title={The role of coverage in online reinforcement learning},
  author={Xie, Tengyang and Foster, Dylan J and Bai, Yu and Jiang, Nan and Kakade, Sham M},
  journal={arXiv preprint arXiv:2210.04157},
  year={2022}
}

@inproceedings{coverage_3,
  title={What can online reinforcement learning with function approximation benefit from general coverage conditions?},
  author={Liu, Fanghui and Viano, Luca and Cevher, Volkan},
  booktitle={International Conference on Machine Learning},
  pages={22063--22091},
  year={2023},
  organization={PMLR}
}

@article{coverage_4,
  title={When to trust your simulator: Dynamics-aware hybrid offline-and-online reinforcement learning},
  author={Niu, Haoyi and Qiu, Yiwen and Li, Ming and Zhou, Guyue and Hu, Jianming and Zhan, Xianyuan and others},
  journal={Advances in Neural Information Processing Systems},
  volume={35},
  pages={36599--36612},
  year={2022}
}

@ARTICLE{Kartik_NM22,
  author={Kartik, Dhruva and Nayyar, Ashutosh and Mitra, Urbashi},
  journal={IEEE Transactions on Automatic Control}, 
  title={Fixed-Horizon Active Hypothesis Testing}, 
  year={2022},
  volume={67},
  number={4},
  pages={1882-1897},
  doi={10.1109/TAC.2021.3090742}}

@ARTICLE{BJM25,
  author={Bozkus, Talha and Javidi, Tara and Mitra, Urbashi},
  journal={IEEE Transactions on Signal Processing}, 
  title={Coverage Analysis for Digital Cousin Selection — Improving Multi-Environment Q-Learning}, 
  year={2025},
  volume={73},
  number={},
  pages={3843-3856},
  doi={10.1109/TSP.2025.3598253}}

@inproceedings{tarbouriech2020active,
  title={Active model estimation in markov decision processes},
  author={Tarbouriech, Jean and Shekhar, Shubhanshu and Pirotta, Matteo and Ghavamzadeh, Mohammad and Lazaric, Alessandro},
  booktitle={Conference on Uncertainty in Artificial Intelligence},
  pages={1019--1028},
  year={2020},
  organization={PMLR}
}

@inproceedings{hazan2019provably,
  title={Provably efficient maximum entropy exploration},
  author={Hazan, Elad and Kakade, Sham and Singh, Karan and Van Soest, Abby},
  booktitle={International Conference on Machine Learning},
  pages={2681--2691},
  year={2019},
  organization={PMLR}
}

@inproceedings{tarbouriech2019active,
  title={Active exploration in markov decision processes},
  author={Tarbouriech, Jean and Lazaric, Alessandro},
  booktitle={The 22nd International Conference on Artificial Intelligence and Statistics},
  pages={974--982},
  year={2019},
  organization={PMLR}
}

@article{xiasome,
  title={Some Results on Max-min Fairness for Resource Sharing},
  author={Xia, Ye}
}

@inproceedings{tiapkin2023fast,
  title={Fast rates for maximum entropy exploration},
  author={Tiapkin, Daniil and Belomestny, Denis and Calandriello, Daniele and Moulines, Eric and Munos, Remi and Naumov, Alexey and Perrault, Pierre and Tang, Yunhao and Valko, Michal and Menard, Pierre},
  booktitle={International Conference on Machine Learning},
  pages={34161--34221},
  year={2023},
  organization={PMLR}
}

@inproceedings{jin2020reward,
  title={Reward-free exploration for reinforcement learning},
  author={Jin, Chi and Krishnamurthy, Akshay and Simchowitz, Max and Yu, Tiancheng},
  booktitle={International Conference on Machine Learning},
  pages={4870--4879},
  year={2020},
  organization={PMLR}
}

@inproceedings{shekhar2020adaptive,
  title={Adaptive sampling for estimating probability distributions},
  author={Shekhar, Shubhanshu and Javidi, Tara and Ghavamzadeh, Mohammad},
  booktitle={International Conference on Machine Learning},
  pages={8687--8696},
  year={2020},
  organization={PMLR}
}

@inproceedings{lee1906efficient,
  title={Efficient exploration via state marginal matching, 2020},
  author={Lee, Lisa and Eysenbach, Benjamin and Parisotto, Emilio and Xing, Eric and Levine, Sergey and Salakhutdinov, Ruslan},
  booktitle={URL https://openreview. net/forum},
  year={1906}
}

@article{singh2010intrinsically,
  title={Intrinsically motivated reinforcement learning: An evolutionary perspective},
  author={Singh, Satinder and Lewis, Richard L and Barto, Andrew G and Sorg, Jonathan},
  journal={IEEE Transactions on Autonomous Mental Development},
  volume={2},
  number={2},
  pages={70--82},
  year={2010},
  publisher={IEEE}
}

@inproceedings{singh2009rewards,
  title={Where do rewards come from},
  author={Singh, Satinder and Lewis, Richard L and Barto, Andrew G},
  booktitle={Proceedings of the annual conference of the cognitive science society},
  pages={2601--2606},
  year={2009},
  organization={Cognitive Science Society}
}

@inproceedings{kaelbling1993learning,
  title={Learning to achieve goals},
  author={Kaelbling, Leslie Pack},
  booktitle={IJCAI},
  volume={2},
  pages={1094--8},
  year={1993}
}

@article{zheng2018learning,
  title={On learning intrinsic rewards for policy gradient methods},
  author={Zheng, Zeyu and Oh, Junhyuk and Singh, Satinder},
  journal={Advances in neural information processing systems},
  volume={31},
  year={2018}
}

@inproceedings{zhang2021exploration,
  title={Exploration by maximizing R{\'e}nyi entropy for reward-free RL framework},
  author={Zhang, Chuheng and Cai, Yuanying and Huang, Longbo and Li, Jian},
  booktitle={Proceedings of the AAAI Conference on Artificial Intelligence},
  volume={35},
  number={12},
  pages={10859--10867},
  year={2021}
}

@article{ho2016generative,
  title={Generative adversarial imitation learning},
  author={Ho, Jonathan and Ermon, Stefano},
  journal={Advances in neural information processing systems},
  volume={29},
  year={2016}
}

@inproceedings{schulman2015trust,
  title={Trust region policy optimization},
  author={Schulman, John and Levine, Sergey and Abbeel, Pieter and Jordan, Michael and Moritz, Philipp},
  booktitle={International conference on machine learning},
  pages={1889--1897},
  year={2015},
  organization={PMLR}
}

@article{mutti2023convex,
  title={Convex reinforcement learning in finite trials},
  author={Mutti, Mirco and De Santi, Riccardo and De Bartolomeis, Piersilvio and Restelli, Marcello},
  journal={Journal of Machine Learning Research},
  volume={24},
  number={250},
  pages={1--42},
  year={2023}
}

@article{de2024global,
  title={Global reinforcement learning: Beyond linear and convex rewards via submodular semi-gradient methods},
  author={De Santi, Riccardo and Prajapat, Manish and Krause, Andreas},
  journal={arXiv preprint arXiv:2407.09905},
  year={2024}
}

@article{shekhar2021adaptive,
  title={Adaptive sampling for minimax fair classification},
  author={Shekhar, Shubhanshu and Fields, Greg and Ghavamzadeh, Mohammad and Javidi, Tara},
  journal={Advances in Neural Information Processing Systems},
  volume={34},
  pages={24535--24544},
  year={2021}
}

@inproceedings{halim2025fairness,
  title={Fairness-Aware Active Online Learning with Changing Environments},
  author={Halim, Sadaf MD and Zhao, Chen and Wu, Xintao and Khan, Latifur and Grant, Christan Earl and Chen, Feng},
  booktitle={2025 IEEE 41st International Conference on Data Engineering (ICDE)},
  pages={3821--3834},
  year={2025},
  organization={IEEE}
}

@inproceedings{al2023active,
  title={Active coverage for pac reinforcement learning},
  author={Al-Marjani, Aymen and Tirinzoni, Andrea and Kaufmann, Emilie},
  booktitle={The Thirty Sixth Annual Conference on Learning Theory},
  pages={5044--5109},
  year={2023},
  organization={PMLR}
}

@inproceedings{menard2021fast,
  title={Fast active learning for pure exploration in reinforcement learning},
  author={M{\'e}nard, Pierre and Domingues, Omar Darwiche and Jonsson, Anders and Kaufmann, Emilie and Leurent, Edouard and Valko, Michal},
  booktitle={International Conference on Machine Learning},
  pages={7599--7608},
  year={2021},
  organization={PMLR}
}

@inproceedings{kaufmann2021adaptive,
  title={Adaptive reward-free exploration},
  author={Kaufmann, Emilie and M{\'e}nard, Pierre and Domingues, Omar Darwiche and Jonsson, Anders and Leurent, Edouard and Valko, Michal},
  booktitle={Algorithmic Learning Theory},
  pages={865--891},
  year={2021},
  organization={PMLR}
}

\vspace{12pt}
% \color{red}
% IEEE conference templates contain guidance text for composing and formatting conference papers. Please ensure that all template text is removed from your conference paper prior to submission to the conference. Failure to remove the template text from your paper may result in your paper not being published.

\newpage
\appendix
\subsection{Proof of Lemma~\ref{lem:smooth_Urho}}
\label{appendix: lem_smoothUrho_proof}
\begin{proof}
We compute the Hessian of $U_\rho$ and bound its operator norm on $\mathcal{D}_{\underline{\eta}}$.
Since $U_\rho$ is separable across coordinates $d_{s, a}$, its Hessian is diagonal.

For each $(s,a)$,
\[
\frac{\partial U_\rho(d)}{\partial d_{s, a}}
=\frac{\mu_{s, a}^{\rho}}{1-\rho}(1-\rho)d_{s, a}^{-\rho}
= \mu_{s, a}^{\rho}d_{s, a}^{-\rho},
\]
and
\[
\frac{\partial^2 U_\rho(d)}{\partial d_{s, a}^2}
= -\rho\,\mu_{s, a}^{\rho}\,d_{s, a}^{-(\rho+1)}.
\]
Thus,
\begin{equation} \notag
    \begin{split}
        &\nabla^2 U_\rho(d)
        =
        \mathrm{diag}\!\left(-\rho\,\mu_{s, a}^{\rho}\,d_{s, a}^{-(\rho+1)}\right)\\
        \rightarrow
        &\|\nabla^2 U_\rho(d)\|_{\mathrm{op}}
        =
        \max_{s,a}\rho\,\mu_{s, a}^{\rho}\,d_{s, a}^{-(\rho+1)}.
    \end{split}
\end{equation}

Using again $d_{s, a}\ge 2\underline{\eta}$ for $d\in\mathcal{D}_{\underline{\eta}}$ yields
\[
\|\nabla^2 U_\rho(d)\|_{\mathrm{op}}
\le
\max_{s,a}\frac{\rho\,\mu_{s, a}^{\rho}}{(2\underline{\eta})^{\rho+1}}
=
\frac{\rho\,\mu_{\max}^{\rho}}{(2\underline{\eta})^{\rho+1}}.
\]

Finally, for any twice differentiable function, $C$-smoothness w.r.t.\ $\|\cdot\|_2$ is implied by
$\sup_{d\in\mathcal{D}_{\underline{\eta}}}\|\nabla^2 U_\rho(d)\|_{\mathrm{op}}\le C$.
Applying the bounds above completes the proof.
\end{proof}

\subsection{Proof of Theorem~\ref{thm: FW-oracle-regret}}
\label{appendix:proof_oracle_algo}
\subsubsection{Prerequisites}
\begin{lemma}[Bound of empirical occupancy]
\label{lma:bound_empirical_occupancy}
Let $\pi$ be a stationary policy inducing an ergodic and reversible chain $P_\pi$ with
spectral gap $\gamma_\pi$ and stationary distribution $\psi_\pi$. For any budget $n > 0$ and
state $s\in \mathcal{S}$,

Fix $\delta\in(0,1)$. With probability at least $1-\delta$, the empirical occupancy measure
$\widehat\psi_n(s,a) := \frac{T_{s, a}(n)}{n}$, 
where $T_{s,a}(n)$ denotes the number of times action $a$ is taken in state $s$
over $n$ samples, satisfies the following bound for all $(s,a)\in\mathcal{S}\times\mathcal{A}$.
\begin{equation}\label{eq:occ_conc}
\begin{split}
    &\big|\widehat\psi_n(s,a)-\psi_\pi(s,a)\big| \\
    &\;\le\;
    \frac{\sqrt{2B/\gamma_{\min}}+\sqrt{\ln(4SA/\delta)/2}}{\sqrt{n}}
    \;+\;
    \frac{20B}{\gamma_{\min}n},
\end{split}
\end{equation}
where $S:=|\mathcal S|$, $A:=|\mathcal A|$, and
\begin{equation}\label{eq:B_def}
B \;=\; \log\!\left(\frac{S^2A^{S+1}}{\delta/2}\sqrt{\frac{1}{\underline{\eta}}}\right).
\end{equation}
\end{lemma}

\begin{proof}
Fix a stationary policy $\pi$.  
Define the empirical state and state--action occupancies as
\[
\widehat\psi_n(s) := \frac{T_{s}(n)}{n},
\qquad
\widehat\psi_n(s,a) := \frac{T_{s, a}(n)}{n},
\]
and let the stationary state--action distribution be
\[
\psi_\pi(s,a) := \psi_\pi(s)\,\pi(a\mid s).
\]

From Lemma 4 in \cite{tarbouriech2019active}, we have that with probability at least $1-\delta$,
\[
\big|\widehat\psi_n(s) - \psi_\pi(s)\big| \le f(n, \delta) = \sqrt{ \frac{2B}{\gamma_{\min} n} }
\;+\;
\frac{20B}{\gamma_{\min}n}, \quad \forall s\in \mathcal{S}.
\]

Using the identity
\begin{equation}
    \begin{split}
        \widehat\psi_n(s,a) &= \widehat\psi_n(s)\,\widehat{\pi}_n(a\mid s),\\
        \widehat{\pi}_n(a\mid s) &:= \frac{T_{s, a}(n)}{T_{s}(n)} \;\; \text{(when $T_{s}(n)>0$)},
    \end{split}
\end{equation}
we obtain
\begin{align*}
&\big|\widehat\psi_n(s,a) - \psi_\pi(s,a)\big|\\
=&\big|\widehat\psi_n(s)\widehat{\pi}_n(a\mid s) - \psi_\pi(s)\pi(a\mid s)\big| \\
\le&\big|\widehat\psi_n(s) - \psi_\pi(s)\big|\,\pi(a\mid s)
+
\widehat\psi_n(s)\,\big|\widehat{\pi}_n(a\mid s) - \pi(a\mid s)\big|.
\end{align*}

The first term is controlled by the assumed state occupancy bound.
For the second term, the actions taken at
state $s$ are i.i.d.\ according to $\pi(\cdot\mid s)$, and Hoeffding's
inequality yields that for any $\delta\in(0,1)$,
\[
\Pr\!\left(
\big|\widehat{\pi}_n(a\mid s) - \pi(a\mid s)\big|
\ge
\sqrt{\frac{\ln(2/\delta)}{2T_{s}(n)}}
\right)
\le \delta.
\]

Combining the bounds, we conclude that with
probability at least $1-2\delta$,
\begin{align}\notag
    \big|\widehat\psi_n(s,a) - \psi_\pi(s,a)\big|
    &\;\le\;
    \pi(a\mid s)\,f(n, \delta)
    +
    \widehat\psi_n(s)\sqrt{\frac{\ln(2/\delta)}{2\,T_{s}(n)}} \\ \notag
    &\;\le\;
    f(n, \delta)
    +
    \widehat\psi_n(s)\sqrt{\frac{\ln(2/\delta)}{2\,n \widehat\psi_n(s)}} \\ \notag
    &=
    f(n, \delta)
    +
    \sqrt{\frac{\ln(2/\delta)\widehat\psi_n(s)}{2\,n }} \\ \notag
    &\;\le\;
    f(n, \delta)
    +
    \sqrt{\frac{\ln(2/\delta)}{2\,n }}
\end{align}

Hence, for a single state-action pair, with probability at least $1-\delta$,
\begin{align}
\notag
    \big|\widehat\psi_n(s,a) - \psi_\pi(s,a)\big|
    &\;\le\;
    \sqrt{ \frac{2B}{\gamma_{\min} n} }
\;+\;
\frac{20B}{\gamma_{\min}n}
    +
    \sqrt{\frac{\ln(4/\delta)}{2\,n }}\\ \notag
    &= \frac{\sqrt{2B/\gamma_{\min}} + \sqrt{\ln(4/\delta)/2}}{\sqrt{n}} + \frac{20B}{\gamma_{\min}n}
\end{align}
where
$$B = \log\!\left( \frac{SA^{S}}{\delta/2} \sqrt{\frac{1}{\underline{\eta}}} \right).$$

We need to take a union bound over state action pairs which leads to $\delta = \frac{\delta'}{SA}$ in the high-probability guarantees.

Therefore, with probability at least $1-\delta$,
\begin{equation}
    \begin{split}
        &\big|\widehat\psi_n(s,a) - \psi_\pi(s,a)\big| \\
    \le &\frac{\sqrt{2B/\gamma_{\min}} + \sqrt{\ln(4SA/\delta)/2}}{\sqrt{n}} + \frac{20B}{\gamma_{\min}n}, \quad
    \forall (s, a) \in \mathcal{S}\times\mathcal{A}
    \end{split}
\end{equation}
where
$$B = \log\!\left( \frac{S^2A^{S+1}}{\delta/2} \sqrt{\frac{1}{\underline{\eta}}} \right).$$
\end{proof}

\subsubsection{Core of Proof}

\begin{proof}

We denote by $\xi_{k+1}$ the approximation error at the end of each episode $k$
(i.e., at time $t_{k+1}-1$). Recalling that 
$\beta_k = \tau_k / (t_{k+1}-1)$, we have
\begin{equation} \notag
    \begin{split}
        \xi_{k+1}
        &= U_\rho(d^\star_\rho) - U_\rho(\widehat{d}_{k+1}) \\
        &= U_\rho(d^\star_\rho) - U_\rho\!\left((1-\beta_k)\widehat{d}_k + \beta_k \widehat{\psi}_{k+1}\right),
    \end{split}
\end{equation}
where $d_\rho^\star \in \arg\min_{d\in \mathcal{D}} U_\rho(d)$, and $\widehat{d}_k$ abbreviates $\widehat{d}(t_k)$.

We have the following series of inequalities:
\[
\begin{aligned}
&\xi_{k+1} \\
&\le 
U_\rho(d_\rho^\star) - U_\rho(\widehat{d}_k) 
- \beta_k \langle \nabla U_\rho(\widehat{d}_k),
    \widehat{\psi}_{k+1} - \widehat{d}_k \rangle
  + C_{\underline{\eta}} \beta_k^2
\\[0.4em]
&=
U_\rho(d_\rho^\star) - U_\rho(\widehat{d}_k) 
- \beta_k \langle \nabla U_\rho(\widehat{d}_k),
    \psi_{k+1} - \widehat{d}_k \rangle
  + C_{\underline{\eta}} \beta_k^2 \\
&\qquad \qquad - \beta_k \langle \nabla U_\rho(\widehat{d}_k),
    \widehat{\psi}_{k+1} - \psi_{k+1} \rangle
\\[0.4em]
&\le 
U_\rho(d_\rho^\star) - U_\rho(\widehat{d}_k) 
- \beta_k \langle \nabla U_\rho(\widehat{d}_k),
    d_\rho^\star - \widehat{d}_k \rangle
  + C_{\underline{\eta}} \beta_k^2 \\
&\qquad \qquad - \beta_k \langle \nabla U_\rho(\widehat{d}_k),
    \widehat{\psi}_{k+1} - \psi_{k+1} \rangle
\\[0.4em]
&\le 
(1-\beta_k)\xi_k
+ C_{\underline{\eta}} \beta_k^2
+ \beta_k
\underbrace{
   \langle \nabla U_\rho(\widehat{d}_k),
      \psi_{k+1} - \widehat{\psi}_{k+1}  \rangle
}_{\Delta_{k+1}},
\end{aligned}
\]

where the first step follows from the $C_{\underline{\eta}}$-smoothness of $U_\rho$, proved in Lemma~\ref{lem:smooth_Urho}, 
the second inequality comes from the FW optimization step and the definition 
of $\psi_{k+1}$, which gives 
$\langle \nabla U_\rho(\widehat{d}_k),\, 
\psi_{k+1} - \widehat{d}_k \rangle 
\ge 
\langle \nabla U_\rho(\widehat{d}_k),\, 
d_\rho^\star - \widehat{d}_k \rangle$,
the final step follows from the concavity of $U_\rho$. 
The term $\Delta_{k+1}$ refers to the discrepancy between 
the stationary state--action distribution 
$\psi_{k+1}$ and the empirical frequency $\widehat{\psi}_{k+1}$ 
of its realization for $\tau_k$ steps.

If we optimize $\langle \nabla U_\rho(\widehat{d}_k),\, 
d \rangle$ in $d \in \mathcal{D}_{\underline{\eta}}$
we deduce the important property that for any $\delta \in (0,1)$, there exists an 
episode $k_\delta$ such that for all episodes $k$ succeeding it (and including it), 
we have with probability at least $1-\delta$
\[
\widehat{d}_{s,a}(t_k) \ge \underline{\eta},
\qquad \forall s \in \mathcal{S}, a \in \mathcal{A}, \ \forall k \ge k_\delta.
\]

For any $k \ge k_\delta$, we can write
\begin{equation} \notag
    \begin{split}
        \langle \nabla U_\rho(\widehat{d}_k), &
              \psi_{k+1} - \widehat{\psi}_{k+1} \rangle \\
        &=
        \sum_{s, a}
        \big(\frac{\mu_{s, a}}
             {\widehat{d}_{s, a}(k)} \big)^\rho
        \left( \psi_{k+1}(s,a) - \widehat{\psi}_{k+1}(s,a) \right) \\
        &\;\;\le\;\;
        SA
        \big(\frac{\mu_{\max}}
             {\underline{\eta}} \big)^\rho
        \left\|
        \psi_{k+1}(s, a) - 
         \frac{\nu_{k+1}(s, a)}{\tau_{k}}
        \right\|_{\infty}.
    \end{split}
\end{equation}

% Let 
% \[
% B = \log\!\left( \frac{SA^{S}}{\delta} \sqrt{\frac{1}{\eta}} \right).
% \]
% From Lemma~4 in \cite{tarbouriech2019active}, we have with probability at least $1-\delta$
% simultaneously for every state $s$ and every policy followed during 
% the episode,
% \[
% \left|
%   \eta_{g_{k+1}}(s) - \frac{\nu_{k+1}(s)}{\tau_{k}} 
% \right|
% \;\le\;
% \sqrt{ \frac{2B}{\gamma_{\min} \tau_k} }
% \;+\;
% \frac{20B}{\gamma_{\min}\tau_k}.
% \]

From Lemma~\ref{lma:bound_empirical_occupancy}, we have with probability at least $1-\delta$,
\begin{equation}\notag
    \begin{split}
        &\left\|
        \psi_{k+1}(s, a) - 
         \frac{\nu_{k+1}(s, a)}{\tau_{k}}
        \right\|_{\infty} \\
        &\le \frac{\sqrt{2B/\gamma_{\min}} + \sqrt{\ln(4SA/\delta)/2}}{\sqrt{\tau_k}} + \frac{20B}{\gamma_{\min}\tau_k}.
    \end{split}
\end{equation}

Here 
\[
B = \log\!\left( \frac{S^2A^{S+1}}{\delta/2} \sqrt{\frac{1}{\underline{\eta}}} \right).
\]
Hence we obtain the following bound on $\Delta_{k+1}$, for $k \ge k_\delta$ with probability at least $1-\delta$:
\[
\Delta_{k+1}
\;\le\;
\frac{c_1}{\sqrt{\tau_k}}
\;+\;
\frac{c_2}{\tau_k},
\]
\[
\text{with}\quad
\begin{cases}
c_1 = SA
        \big(\frac{\mu_{\max}}
             {\underline{\eta}} \big)^\rho
       \big(\sqrt{2B/\gamma_{\min}} + \sqrt{\ln(4SA/\delta)/2}\big),\\[1.2em]
c_2 = SA
        \big(\frac{\mu_{\max}}
             {\underline{\eta}} \big)^\rho
       \dfrac{20B}{\gamma_{\min}},
\end{cases}
\]

which provides the bound for $k \ge k_\delta$
\begin{equation}
\xi_{k+1}
\;\le\;
(1 - \beta_k)\xi_k
\;+\;
\beta_k \Bigl(
   \frac{c_1}{\sqrt{\tau_k}}
   +
   \frac{c_2}{\tau_k}
\Bigr)
\;+\;
C_\eta \beta_k^2.
\label{eq:xi_recursion}
\end{equation}

Choosing episode lengths $\tau_k = \tau_1 k^2$, which yields 

\[
t_k = \tau_1\sum_{i = 0}^{k-1}i^2 + 1 
= \tau_1 \frac{(k-1)k(2k-1)}{6} + 1,
\]
\[
\beta_k
=
\frac{\tau_k}{t_{k+1} - 1}
=
\frac{6k^2}{k(k+1)(2k+1)}
\in
\Bigl[\frac{1}{k},\, \frac{3}{k}\Bigr].
\]

Consequently we get
\begin{equation}\notag
\begin{split}
    &\beta_k \left(
\frac{c_1}{\sqrt{\tau_k}}
+
\frac{c_2}{\tau_k}
\right)
+
C_\eta \beta_k^2
\;\le\;
\frac{b_\delta}{k^2} \\
&\text{with}\quad
b_\delta = \frac{3c_1}{\sqrt{\tau_1}}
+ \frac{3c_2}{\tau_1 k_\delta}
+ 9 C_\eta .
\end{split}
\end{equation}

Hence the recurrence inequality \eqref{eq:xi_recursion} becomes
\begin{equation}
\xi_{k+1}
\;\le\;
\left(1 - \frac{1}{k}\right)\xi_k
\;+\;
\frac{b_\delta}{k^2}.
\notag
\end{equation}

We pick an integer $q$ such that $\xi_q \ge 0$ is satisfied.\footnote{Immediate induction guarantees the positivity of the sequence $(u_n)$.}
We define the sequence $(u_n)_{n \ge q}$ as $u_q = \xi_q$ and
\[
u_{n+1}
=
\left(1 - \frac{1}{n}\right)u_n
+ \frac{b_\delta}{n^2}.
\]

By rearranging, we get
\[
(n+1) u_{n+1} - n u_n
=
-\,\frac{u_n}{n}
+
\frac{b_\delta(n+1)}{n^2}
\;\le\;
\frac{b_\delta(n+1)}{n^2}.
\]

By telescoping, we obtain
\[
n u_n - q u_q
\;\le\;
2 b_\delta \sum_{i=q}^{n-1} \frac{1}{i}
\;\le\;
2 b_\delta \log\!\Bigl(\frac{n-1}{q-1}\Bigr).
\]
From $t_k = \tau_1 \frac{(k-1)k(2k-1)}{6} + 1$ and $2k^3 - 3k^2 + k < 2k^3$ for $k \ge 1$, we get
$$k \ge \big(\frac{3(t_k-1)}{\tau_1}\big)^{1/3}$$
Let $K \ge k_\delta$.  
We thus have with probability at least $1 -\delta$
\[
\xi_k
\le
\frac{q\xi_q + 2b_\delta \log K}{K}
\le
\frac{\tau_1^{1/3}}{3^{1/3}(t_k-1)^{1/3}}
\left(
q\xi_q + 2b_\delta \log K 
\right).
\]

We conclude the desired high-probability bound $\xi_k = \widetilde{O}(1/t_K^{1/3})$.
    
\end{proof}

\end{document}